\documentclass[times,final,authoryear]{article}
\usepackage{amsmath}
  \usepackage{paralist}
  \usepackage{graphics} 
\usepackage{graphicx}  
 \usepackage[font=footnotesize,labelfont=footnotesize]{caption}
\usepackage[font=scriptsize,labelfont=scriptsize]{subcaption}

\usepackage{algorithmic}

\usepackage{amssymb}
\usepackage[all,arc]{xy}
\usepackage{mathrsfs}
\usepackage{algorithm2e}
\usepackage{amsbsy}
\usepackage{amsfonts}
\usepackage{amsthm}
\usepackage{array}
\usepackage{calrsfs}

\usepackage{bbm} 
\usepackage[font=scriptsize,labelfont=bf]{caption}
\usepackage{color}
\usepackage[mathscr]{eucal}
\usepackage{mathrsfs}
\usepackage{mathtools}

\usepackage{float}
\usepackage{stmaryrd}
\usepackage[font=scriptsize,labelfont=scriptsize]{subcaption}
\usepackage{graphicx}
\usepackage{url}
\usepackage{verbatim}
\usepackage{adjustbox}
\usepackage{multirow,array,framed}

\usepackage{url}

\newcommand{\argmax}{\text{argmax}}

\newtheorem{thm}{Theorem}[section]
\newtheorem{defn}[thm]{Definition}

\begin{document}
\date{}

\title{Spatially regularized active diffusion learning for high-dimensional images}

\author{James M. Murphy\\
Department of Mathematics\\
Tufts University\\
Medford, MA 02155, USA\\
\texttt{jm.murphy@tufts.edu}
}

\maketitle

\bigskip

\begin{abstract}
An active learning algorithm for the classification of high-dimensional images is proposed in which spatially-regularized nonlinear diffusion geometry is used to characterize cluster cores.  The proposed method samples from estimated cluster cores in order to generate a small but potent set of training labels which propagate to the remainder of the dataset via the underlying diffusion process.  By spatially regularizing the rich, high-dimensional spectral information of the image to efficiently estimate the most significant and influential points in the data, our approach avoids redundancy in the training dataset.  This allows it to produce high-accuracy labelings with a very small number of training labels.  The proposed algorithm admits an efficient numerical implementation that scales essentially linearly in the number of data points under a suitable data model and enjoys state-of-the-art performance on real hyperspectral images.
\end{abstract}

\section{Introduction}

Machine learning is revolutionizing the sciences.  With recent advances in data collection, hardware, and algorithms, automated methods have achieved near-human or better performance on a range of problems.  However, in many cases advances in machine learning are based on new architectures for deep neural networks, which while sometimes extremely effective, require large training sets of labeled data points to predict well.  While large training sets are available in some scientific communities, others do not lend themselves as naturally to this paradigm.  For example, producing the label of a pixel in a remotely sensed image may require deploying a human analyst to the site where the image was captured.  It is thus unrealistically expensive to produce the kind of large, labeled training datasets necessary to deploy many state-of-the-art supervised learning methods.  It is imperative instead to develop methods that achieve strong performance with only a very small number of labeled data points.  

In contrast to \emph{supervised} methods which demand large training sets \cite{Scholkopf2001_Learning, Goodfellow2016_Deep}, \emph{unsupervised} methods \cite{Friedman2001_Elements} and \emph{semi-supervised} \cite{Chapelle2006_Semi} methods require only unlabeled points or a small number of labeled points together with many unlabeled points, respectively.  In particular, \emph{active learning} \cite{Settles2009_Active} is the semi-supervised learning paradigm in which the algorithm may query certain data points for labels; this is in contrast to traditional supervised or semi-supervised algorithms in which the labeled training points are selected uniformly at random.  The crucial questions for active learning are how use the existing labeled and unlabeled data to select the query points, and how to propagate the (potentially very small in number) labeled points to the rest of the unlabeled data.

This article proposes an active learning method for high-dimensional image data with three major innovations.  First, \emph{spatially regularized diffusion geometry} \cite{Coifman2005} is computed in order to efficiently capture the geometry of the data.  The resulting diffusion distances \cite{Coifman2006} characterize similarities in the data not just with respect to the high-dimensional spectral structure of the data, but also with respect to the underlying spatial structure.  The second major innovation consists in \emph{using cluster core inferences derived from the spatially regularized diffusion geometry to query points for labels}.  These labels are then propagated to the rest of the data using the diffusion distances in a spatially-regularized manner.  Unlike many existing algorithms for active learning of high-dimensional images, the proposed \emph{spatially regularized learning by active nonlinear diffusion (SR LAND)} method is cluster-driven, and does not require any form of pre-training or an initial random training set.  This approach is able to efficiently combine the unsupervised structure with a potent set of carefully chosen training points, and can achieve high-accuracy results with a very small number of training labels.  The third major innovation is \emph{an implementation of the proposed method that scales essentially linearly in the number of points in the image under a suitable data model}; this low complexity makes it appropriate for large datasets.  We evaluate SR LAND on real hyperspectral images \cite{Eismann2012_Hyperspectral}, understood both as $n_{1}\times n_{2}\times D$ tensors (with two spatial dimensions $n_{1}$ and $n_{2}$ and a spectral dimension $D$) and as $D$-dimensional point clouds.  A software implementation of the proposed algorithm accompanies this article\footnote{\url{https://jmurphy.math.tufts.edu/Code/}}.

The remainder of this article is organized as follows.  Section \ref{sec:Background} provides background on active learning and on the main mathematical tool of the proposed algorithm: the diffusion geometry of high-dimensional data.  The proposed algorithm is introduced in Section \ref{sec:Algorithm}.  Experimental analysis on real hyperspectral images and comparisons with benchmark and state-of-the-art active learning methods appear in Section \ref{sec:Experiments}.  Theoretical and empirical analysis of the key parameters of the proposed method and its computational complexity and runtime are also in Section \ref{sec:Experiments}.  We conclude and discuss future research directions in Section \ref{sec:Conclusions}.

\section{Background}\label{sec:Background}

In order to introduce our proposed method, we review active learning and its applications to high-dimensional image processing, as well as diffusion geometry.

\subsection{Background on Active Learning}\label{subsec:BackgroundAL}

In order to produce efficient prediction algorithms in the low-training regime, active learning proposes to carefully determine which points should be queried for labels.  Broadly, there are two major approaches to determining which points to query for labels \cite{Dasgupta2011_Two}.  The first approach considers the active learning problem as a supervised learning problem, but with very few labels.  This frames active learning as the question of how to choose the labels to have maximal impact in improving the performance of the classification algorithm.  Thinking of supervised learning as choosing the best among a large hypothesis class of possible classifiers, active learning in this context seeks to winnow down the hypothesis class as rapidly as possible as a function of the number of labeled samples.  In the specific context of support vector machines (SVM), this means querying for labels points which will move the class boundaries most significantly in an error-reducing manner; we refer to this class of active learning methods as \emph{boundary-based active learning methods} \cite{Cohn1994_Improving, Balcan2007_Margin, Castro2008_Minimax, Dasgupta2008_General, Koltchinskii2010_Rademacher, Hanneke2012_Activized, Minsker2012_Plug, Hanneke2015_Minimax, Gelbhart2019_Relationship}.

The second major approach to active learning is to consider it as unsupervised learning, but in which a few training labels will be provided.  This frames active learning as the question of how to use the latent cluster structure in the data to determine which points will have maximum impact in propagating their labels to the rest of the data.  This type of learning is called \emph{cluster-based active learning}, because it leverages cluster structure in the data---learned without supervision---to decide which points to query for labels.  Approaches to cluster-based active learning often construct a hierarchical tree, with, for example, single or average-linkage clustering \cite{Friedman2001_Elements}.  This tree is then queried in an iterative manner---more points are queried from branches that have ambiguous labels, while fewer (or no) additional points are queried from label-homogeneous branches \cite{Dasgupta2008_Hierarchical, Urner2013_PLAL}.  Alternative tree constructions have also been proposed, including those based on manifold learning methods \cite{Maggioni2019_Learning}, which may better capture geometric structures in the data as well as improve robustness to high-dimensional noise.  

In the context of high-dimensional hyperspectral image analysis, a range of active learning approaches have been developed.  Most are boundary-based methods \cite{Rajan2008_Active, Tuia2009_Active, Li2010_Semisupervised, Sun2015_Active, Murphy2018_Iterative}, which typically require an initial set of random training data, or some form of pre-training.  Active learning methods may incorporate additional information particular to images, for example physical knowledge about the sensor or a general spatial regularity constraint.  Indeed, data in high-dimensional images often admit useful regularity: spatially nearby points are generally likely to share the same class label. Incorporating spatial regularity into querying algorithms often improves active learning algorithms for hyperspectral images.  

\subsection{Background on Diffusion Geometry}\label{subsec:BackgroundDG}

A major challenge in analyzing high-dimensional data in $\mathbb{R}^{D}$ is that traditional Euclidean distances cease to be meaningful; this is one manifestation of the curse of dimensionality \cite{Vershynin2018_High}.  However, when high-dimensional data admits low-dimensional structure, dimension-reduction methods and data-dependent distances may glean this structure.  When data lies near a low-dimensional subspace, classical methods such as principal component analysis \cite{Friedman2001_Elements} and Mahalanobis distances \cite{Mahalanobis1936_Generalized} capture the underlying structure in a manner that avoids the curse of dimensionality.  In the case of data that lies near a $d$-dimensional manifold, $d\le D$, \emph{manifold learning methods} may be used to capture this intrinsically low-dimensional but nonlinear structure \cite{Belkin2003_Laplacian, Maaten2008_Visualizing}.  

The \emph{diffusion geometry} of high-dimensional data captures latent, low-dimensional structure through Markovian random walks at different time scales \cite{Coifman2005, Coifman2006}.  For discrete data $X=\{x_{i}\}_{i=1}^{n}\subset\mathbb{R}^{D}$, an undirected, weighted graph is constructed with nodes corresponding to points in $X$ and weights between $x_{i}, x_{j}$ stored in a symmetric weight matrix $W\in\mathbb{R}^{n\times n}$.  The weights are commonly constructed by setting $W_{ij}=\mathcal{K}(x_{i},x_{j})$ for some kernel function $\mathcal{K}:\mathbb{R}^{D}\times\mathbb{R}^{D}\rightarrow [0,\infty)$, typically the Gaussian kernel $\mathcal{K}(x_{i},x_{j})=\mathcal{K}_{\sigma}(x_{i},x_{j})=\exp(-\|x_{i}-x_{j}\|_{2}^{2}/\sigma^{2})$ for some choice of scaling parameter $\sigma>0$.  In order to induce computationally advantageous sparsity in $W$, a point $x_{i}$ may only be connected to its $k$-nearest $\ell^{2}$ neighbors, for some $k=O(\log(n))$.  

Given the weight matrix $W$, a Markov diffusion operator $P\in\mathbb{R}^{n\times n}$ is defined on $X$ by normalizing $W$ to have all rows summing to 1.  Indeed, let $D\in\mathbb{R}^{n\times n}$ be the diagonal degree matrix for $W$: $D_{ii}=\sum_{j=1}^{n}W_{ij}$.  Then $P=D^{-1}W$ is a Markov transition matrix with left multiplication.  Assuming that the underlying graph $\mathcal{G}=(X,W)$ is connected and aperiodic, $P$ has a unique stationary distribution $\pi\in\mathbb{R}^{1\times n}$: $\pi P =\pi$.  

The diffusion distances for the dataset $X$ are derived from $P$, which encodes the latent diffusion geometry of $X$.  For a time $t\ge0$, the \emph{diffusion distance} between $x_{i}$ and $x_{j}$ is \begin{equation}\label{eqn:DiffusionDistances}D_{t}(x_{i},x_{j})=\sqrt{\sum_{k=1}^{n}\left((P^{t})_{ik}-(P^{t})_{jk}\right)^{2}\frac{1}{\pi_{k}}}.\end{equation}Intuitively, points $x_{i}, x_{j}$ are close in diffusion distance at time $t$ if the transition probabilities for $x_{i}$ and $x_{j}$ are similar at time $t$.  Note that $\lim_{t\rightarrow\infty}P^{t}=\mathbbm{1}_{n\times 1}\pi$, so that $D_{t}(x_{i},x_{j})\rightarrow0$ uniformly for all $x_{i},x_{j}$.  For data with low-dimensional cluster structure, a range of times $t$ reveals this cluster structure \cite{Maggioni2019_Learning_JMLR}.

Computing (\ref{eqn:DiffusionDistances}) is inefficient for large data; indeed, computing $D_{t}(x_{i},x_{i})$ for a single pair of points $x_{i}, x_{j}$ is $O(n)$, so computing the full $n\times n$ matrix of pairwise diffusion distances is $O(n^{3})$.  This motivates a spectral formulation of diffusion distances.  Indeed, while $P$ is not symmetric, it is diagonally conjugate to a symmetric matrix: $D^{1/2}PD^{-1/2}=D^{-1/2}WD^{-1/2}$.  In particular, the right eigenvectors of $P$ form an orthonormal basis for the weighted sequence space $\ell^{2}(1/\sqrt{\pi})$, so that diffusion distances may be computed entirely in terms of the (right) eigenvectors $\{\psi_{k}\}_{k=1}^{n}\subset\mathbb{R}^{n\times1}$ of $P$ and the corresponding eigenvalues $\{\lambda_{k}\}_{k=1}^{n}\subset (-1, 1]$, sorted so that $1=\lambda_{1}>|\lambda_{2}|\ge\dots\ge|\lambda_{n}|\ge 0$: \begin{equation}\label{eqn:DiffusionDistancesSpectralFormulation}D_{t}(x_{i},x_{j})=\sqrt{\sum_{k=1}^{n}\lambda_{k}^{2t}\left(\psi_{k}(x_{i})-\psi_{k}(x_{j})\right)^{2}},\end{equation}where $\psi_{k}(x_{i})=(\psi_{k})_{i}$.  

The eigenvectors contribute to the diffusion distances depending on their eigenvalues.  Indeed, since the eigenvalues are sorted in decreasing modulus, for $t$ sufficiently large, only the first small number of terms contribute to (\ref{eqn:DiffusionDistancesSpectralFormulation}).  This motivates truncating the expansion at a suitable $m\le n$: \begin{equation}\label{eqn:DiffusionDistancesSpectralFormulationTruncated}D_{t}(x_{i},x_{j})\approx\sqrt{\sum_{k=1}^{m}\lambda_{k}^{2t}\left(\psi_{k}(x_{i})-\psi_{k}(x_{j})\right)^{2}}.\end{equation}  In addition to reducing the number of eigenvectors necessary to estimate $D_{t}$, removing the eigenvalues close to 0 has the effect of denoising diffusion distances via a low-pass filter on the eigenvectors \cite{Trillos2018_Error}.  In Section \ref{subsec:ComplexityRuntime}, we will show that under suitable assumptions on the data, the eigenvector truncation (\ref{eqn:DiffusionDistancesSpectralFormulationTruncated}) allows all $O(1)$ $D_{t}$-nearest neighbors to be computed with essentially linear complexity in $n$.

\section{Spatially Regularized Learning by Active Nonlinear Diffusion}\label{sec:Algorithm}

The SR LAND method consists of three phases.  First, the unsupervised structure of the data is learned with diffusion geometry, in order to identify homogeneous regions in the data, namely regions that are close in diffusion distance to estimated modes in the data.  Second, the unsupervised structure is used to actively query points that are most likely to significantly impact the data labeling.  Third, these actively sampled training points are propagated to the rest of the dataset according to the underlying diffusion geometry of the data.  
 
The first stage consists in estimating cluster modes in a manner pioneered by the \emph{learning by unsupervised nonlinear diffusion (LUND)} algorithm \cite{Maggioni2019_Learning_JMLR} and its variants \cite{Murphy2018_Diffusion, Murphy2019_Spectral, Murphy2019_Unsupervised}, in which unsupervised clustering is performed by combining empirical density estimation and diffusion geometry.  Indeed, let $X=\{x_{i}\}_{i=1}^{n}\subset\mathbb{R}^{D}$ be a $D$-dimensional image, realized as a point cloud.  A random Markov diffusion is constructed on $X$ as in Section \ref{subsec:BackgroundDG}, but with the important constraint that the underlying weight matrix is a nearest neighbors graph with respect to the \emph{spatial structure} in the data.   This constraint significantly improves the unsupervised clustering of hyperspectral images \cite{Murphy2019_Spectral}.

More precisely, for some radius $r>0$, let $B_{r}(x_{i})$ be the spatial ball around the point $x_{i}$, i.e. the points whose spatial coordinates are within $\ell^{2}$ distance $r$ of the coordinates of $x_{i}$.  Then consider the weight matrix $W$ with \begin{equation}\label{eqn:W}W_{ij}= \begin{cases}
        \exp\left(\frac{-\|x_{i}-x_{j}\|_{2}^{2}}{\sigma^{2}}\right), \ x_{i}\in B_{r}(x_{j}) \text{ or } x_{j}\in B_{r}(x_{i}),\\
       0,  \text{ else.}\\
        \end{cases}\end{equation}The corresponding Markov transition matrix $P=D^{-1}W$ has a unique stationary distribution for $r>0$, and hence we can define diffusion distances as in (\ref{eqn:DiffusionDistances}) and (\ref{eqn:DiffusionDistancesSpectralFormulation}).  
        
Once the diffusion distances are computed, the empirical density of each data point is computed using a kernel density estimator (KDE) $f_{\sigma_{0}}$, with $\sigma_{0}>0$ a bandwidth parameter.  We use the Gaussian KDE, though others could be used.  For a positive integer $k$, we compute for each $x_{i}$ the set of nearest neighbors to $x_{i}$ in the spectral domain (i.e. in $\mathbb{R}^{D}$), $NN_{k}^{\text{spec}}(x_{i}).$  The KDE is then computed over $NN_{k}^{\text{spec}}$: \begin{equation*}\tilde{p}(x_{i})=\sum_{x\in NN_{k}^{\text{spec}}(x_{i})}\exp(-\|x-x_{i}\|_{2}^{2}/\sigma_{0}^{2}).\end{equation*}We normalize the KDE to have $\ell^{1}$ norm 1: $p(x_{i})=\tilde{p}(x_{i})/\left(\sum_{i=1}^{n}\tilde{p}(x_{i})\right)$.  In all experiments, the bandwidth $\sigma_{0}$ is set adaptively to be half the mean distance between all points and their $k$ nearest neighbors, and we use $k=100$ nearest neighbors.

A crucial step of the LUND approach to clustering is to characterize clusters according to their modes, which are both high density and also far in diffusion distance from other points of high density.  Compared to characterizing modes simply in terms of density \cite{Cheng1995_Mean, Ester1996_Density} or in terms of density and being $\ell^{2}$ far from other high density points \cite{Rodriguez2014_Clustering}, the LUND characterization is highly robust to non-convexity and nonlinearity in the underlying cluster structure \cite{Maggioni2019_Learning_JMLR}.  The cluster mode characterization is achieved by considering \begin{equation}\label{eqn:Rho}\tilde{\rho}_{t}(x_{i})=
\begin{cases}
\displaystyle\max_{x\in X}D_{t}(x_{i},x), &\text{ if } x_{i}=\argmax_{y\in X}p(y),\\
\displaystyle\min_{p(x)\ge p(x_{i})} D_{t}(x_{i},x), &\text{ else},
\end{cases}
\end{equation} then normalizing as $\rho_{t}(x_{i})=\tilde{\rho}_{t}(x_{i})/(\max_{1\le i \le n}\tilde{\rho}_{t}(x_{i}))$.  Points with large $\rho_{t}$ values are far in diffusion distance from other high density points, which suggests that if $\rho_{t}(x_{i})$ is large, then it is either a cluster mode, or an outlier which is separated from the rest of the data, but not of value as a cluster exemplar.  

This leads us to consider $\mathcal{D}_{t}(x_{i})=p(x_{i})\rho_{t}(x_{i})$ as a measure of cluster modality, since points which maximize $\mathcal{D}_{t}$ must simultaneously be high density (i.e. not an outlier) and far from other high-density points in a manner which is robust to the underlying geometry in the data.  The estimation of data modes in terms of $\mathcal{D}_{t}$ is the first part of the proposed active learning algorithm; see Algorithm \ref{alg:ModeDetection}.  In the context of active learning, the modes (i.e. the maximizers of $\mathcal{D}_{t}$) are points near cluster centers, which are significant because they characterize cluster cores as the points $D_{t}$-near the modes.  These regions may be confidently classified with a single label, so it is valuable to know what that label is.

\RestyleAlgo{algoruled}
\LinesNumbered
\begin{algorithm}[!htb]
	\caption{\label{alg:ModeDetection}Spatially-Regularized Diffusion Geometric Mode Detection }
	\textbf{Input:} $X=\{x_{i}\}_{i=1}^{n}\subset\mathbb{R}^{D}$ (data), $r\in\mathbb{Z}^{+}$ (spatial radius), $k\in\mathbb{Z}^{+}$ (density nearest neighbors), $m\in\mathbb{Z}^{+}$ (number of eigenvectors), $M\in\mathbb{Z}^{+}$ (number of modes)\\
	Construct $W$ as in (\ref{eqn:W}), then normalize to yield $P$.\\
	Compute the $m$ principal eigenpairs $\{(\lambda_{i},\psi_{i})\}_{i=1}^{m}$ of $P$.\\
	Compute the kernel density estimator $\{p(x_{i})\}_{i=1}^{n}$.\\
	Compute $\{\rho_{t}(x_{i})\}_{i=1}^{n}$ as in (\ref{eqn:Rho}).\\
	Estimate the data modes $\{x_{i}^{*}\}_{i=1}^{M}$ as the $M$ maximizers of $\mathcal{D}_{t}(x_{i})=p(x_{i})\rho_{t}(x_{i})$.\\
	\textbf{Output:} $\{(\lambda_{i},\psi_{i})\}_{i=1}^{m}$ (eigenpairs of $P$), $\{p(x_{i})\}_{i=1}^{n}$ (empirical densities), $\{x_{i}^{*}\}_{i=1}^{M}$ (modes).\\
\end{algorithm}

The second stage of the proposed active learning scheme queries labels from the cluster modes learned in Algorithm \ref{alg:ModeDetection}.  This process is detailed in Algorithm \ref{alg:ActiveSampling}.  An example hyperspectral dataset, together with the sampled active labels, is shown illustrated in Figure \ref{fig:SalinasASampling}.

\begin{algorithm}[!htb]
	\caption{\label{alg:ActiveSampling}Active Sampling Procedure}
	\textbf{Input:} $\{(\lambda_{i},\psi_{i})\}_{i=1}^{m}$ (eigenpairs of $P$), $\{x_{i}^{*}\}_{i=1}^{M}$ (modes), $L\in\mathbb{Z}^{+}$ (budget of active queries), $\mathcal{O}$ (labeling oracle).\\
	Let $\mathcal{L}=\{x_{i}^{*}\}_{i=1}^{L}$.\\
	Consult $\mathcal{O}$ to acquire the labeled data $\{(x_{i},y_{i})\}_{x_{i}\in \mathcal{L}}$.\\
	\textbf{Output:} $\{(x_{i},y_{i})\}_{x_{i}\in \mathcal{L}}$ (labeled training data).\\ 
\end{algorithm}

\begin{figure}[!htb]
\centering
\begin{subfigure}{.475\textwidth}
\includegraphics[width=\textwidth]{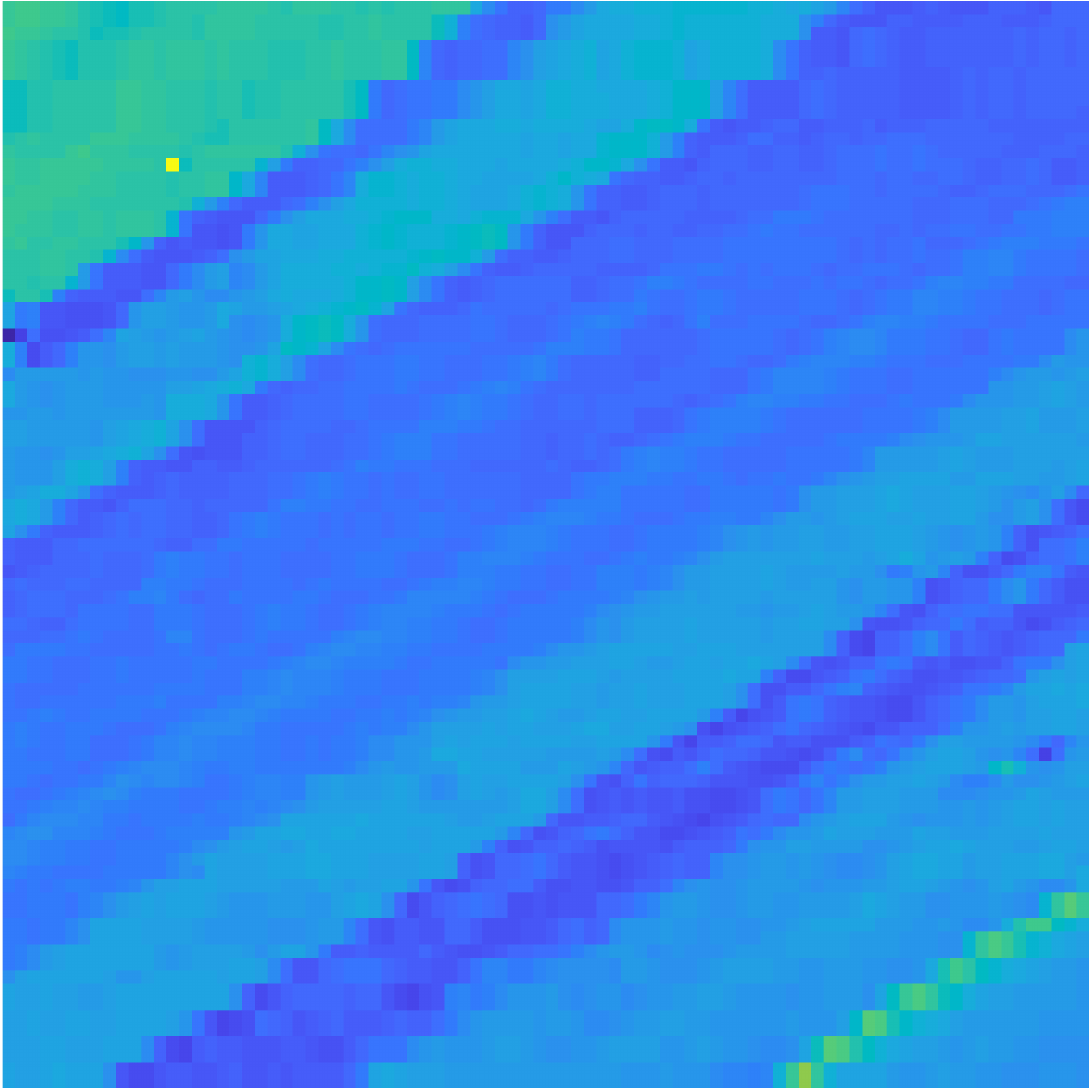}
\subcaption{Salinas A}
\end{subfigure}
\begin{subfigure}{.515\textwidth}
\includegraphics[width=\textwidth]{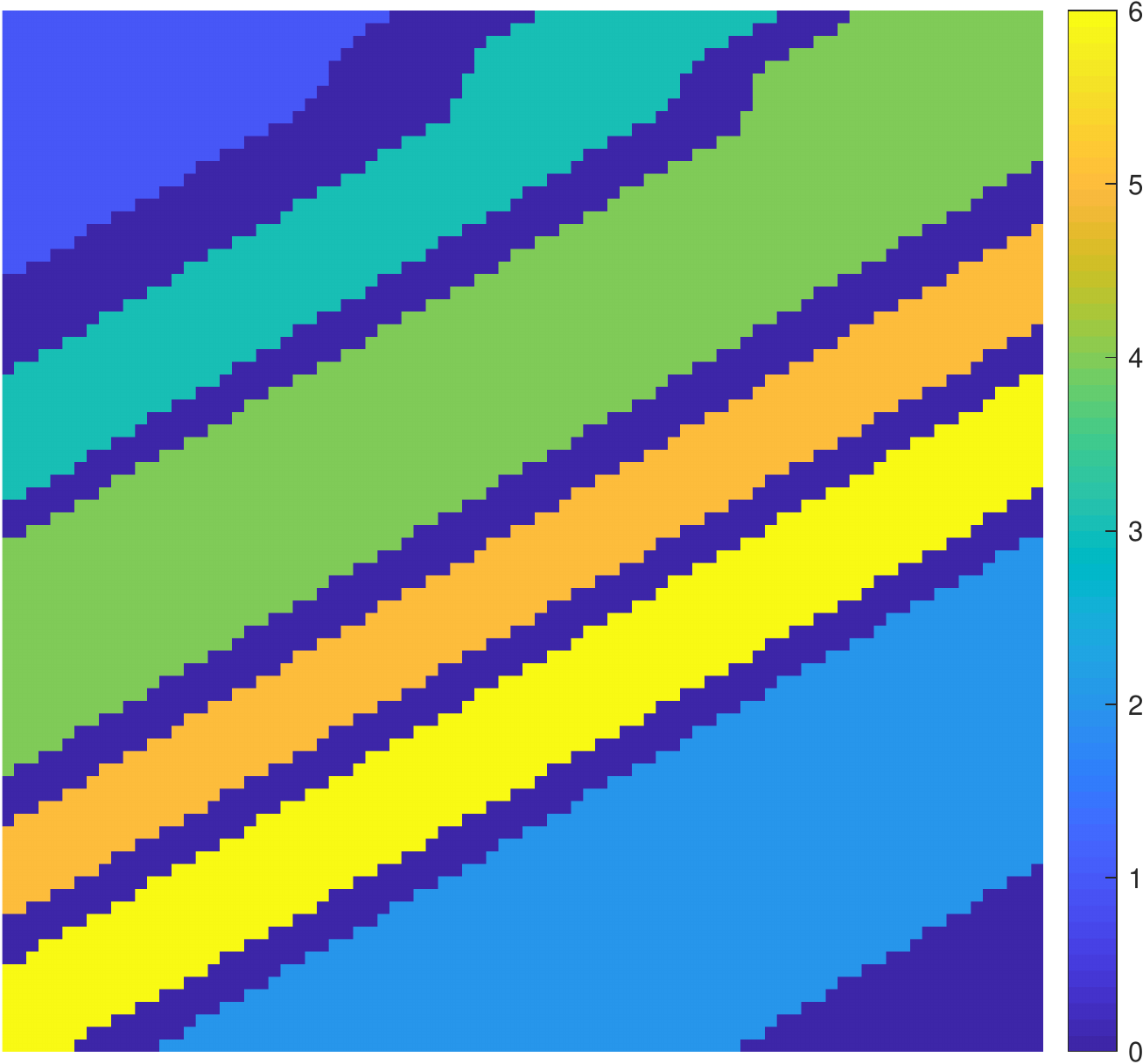}
\subcaption{Salinas A Ground Truth}
\end{subfigure}
\begin{subfigure}{.475\textwidth}
\includegraphics[width=\textwidth]{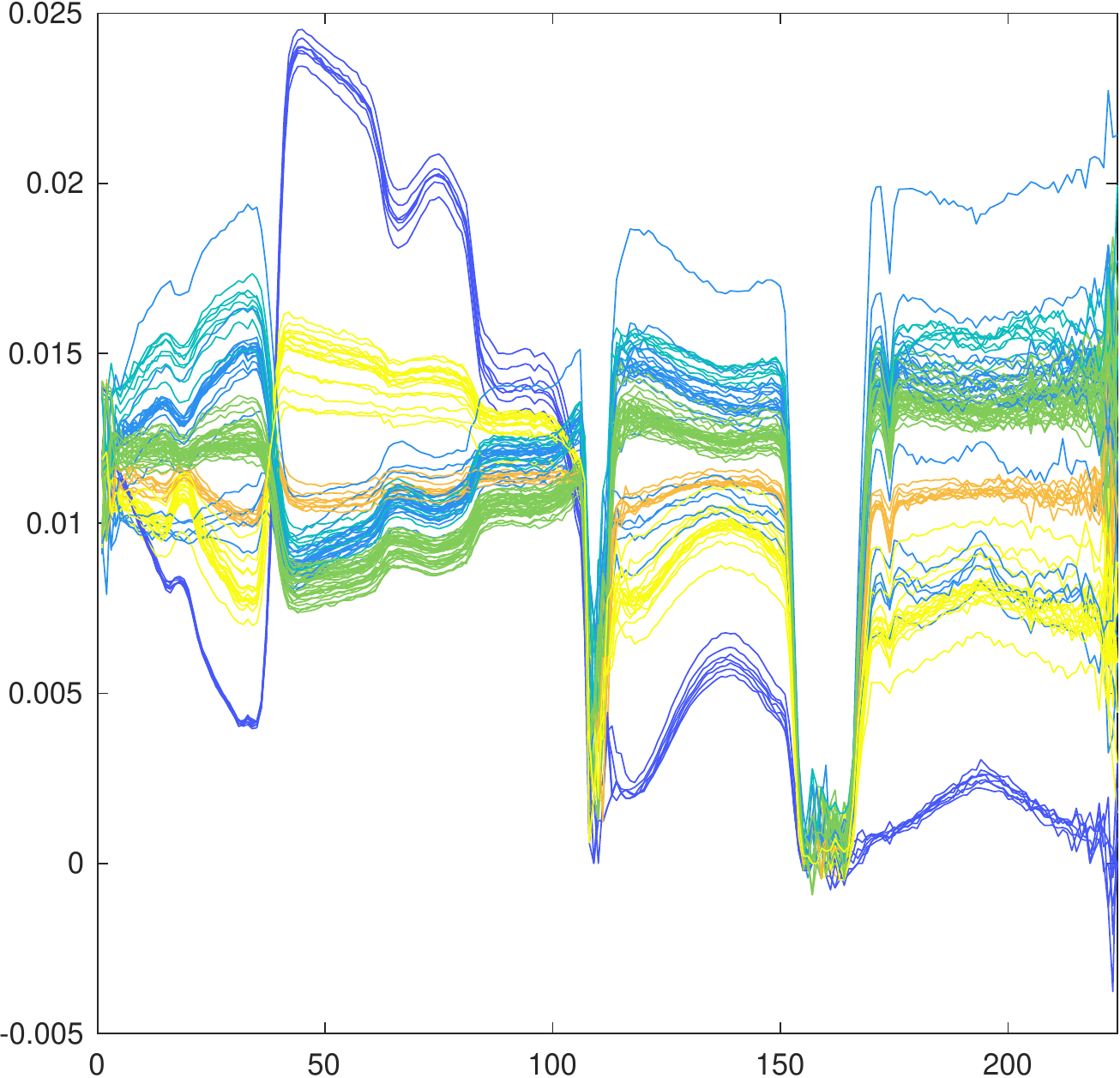}
\subcaption{Salinas A Spectra}
\end{subfigure}
\begin{subfigure}{.515\textwidth}
\includegraphics[width=\textwidth]{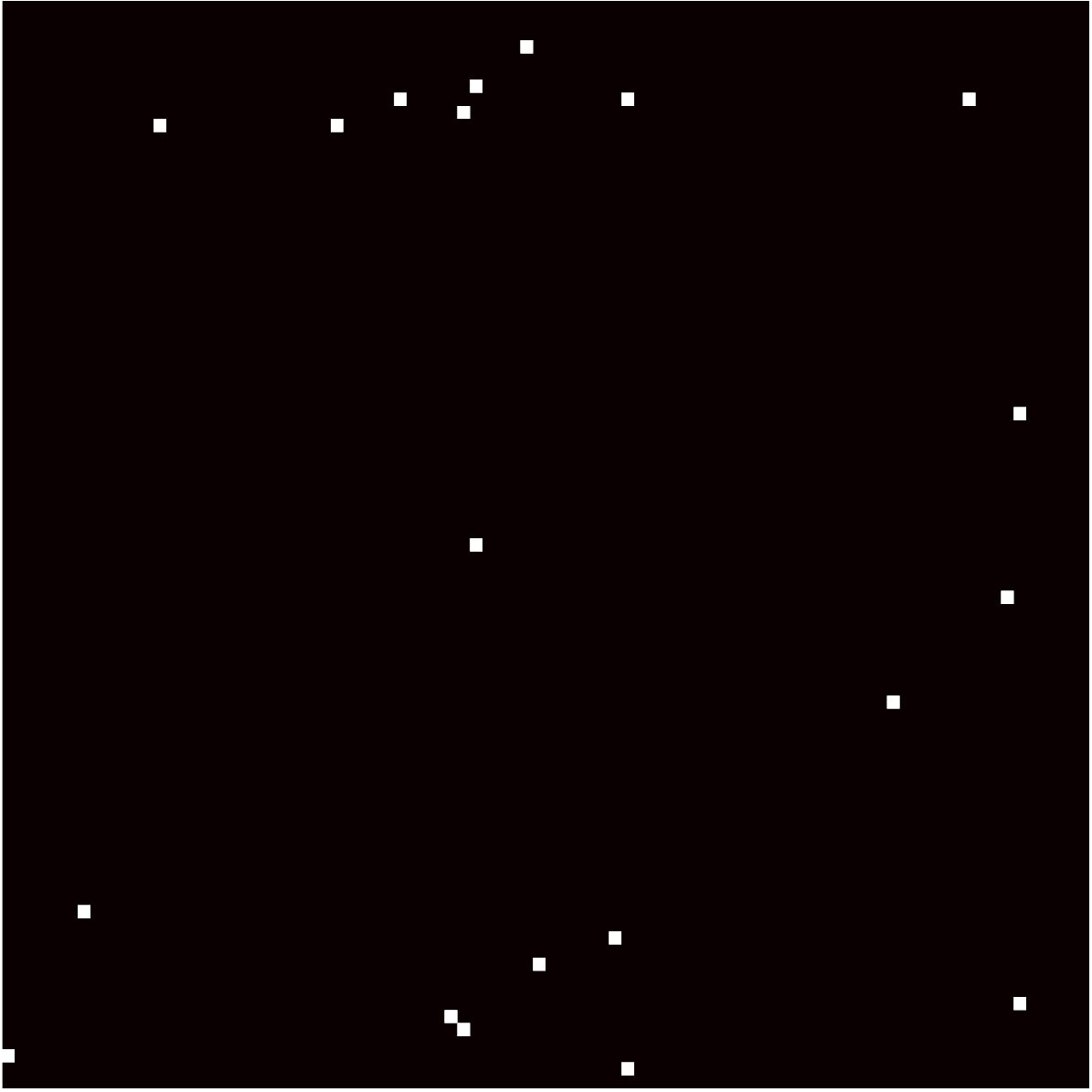}
\subcaption{Salinas A Mode Sampling}
\end{subfigure}
\caption{\label{fig:SalinasASampling}The Salinas A hyperspectral dataset has spatial dimensions $86\times 83$ and consists of 224 spectral bands.  It was collected over Salinas Valley, CA, and has a spatial resolution of 3.7 m/pixel.  In order to differentiate between pixels with identical values, Gaussian noise with mean 0, variance $=10^{-4}$ was added during preprocessing. The sum of all spectral bands is shown in (a).  The full ground truth, arrayed spatially, is shown in (b).   A random sample of 100 spectra for Salinas A are shown in (c) with the same colormap as  (b).  In (d), 20 queries generated from sampling the cluster cores (e.g. the maximizers of $\mathcal{D}_{t}$) are shown.  They are well-spread throughout the image, since they correspond to cluster centers.}
\end{figure}

The third stage of the proposed method is to propagate the queried labels to the rest of the data in a two-step process.  The first step labels all points for which the transition matrix $P$ is confident, while the second step labels the remaining data points in a manner that preserves spatial smoothness in the image.  The notion of confidence is given in terms of the agreement between a point and its nearby spatial neighbors---if the purported label of a datapoint is in agreement with its near neighbors, then the label is confident and is retained.  Otherwise, it is not labeled in the first stage \cite{Murphy2019_Unsupervised}.  The labeling procedure is performed pointwise in order of decreasing density, so that points early in the labeling sequence (i.e. those with large $p$ value) are confidently labeled by default.  Indeed, early in the labeling sequence there are not many spatially nearby points with labels that can contradict the purported label.  The second step straightforwardly labels the remaining ambiguous points so as to preserve spatial regularity.  This third and final stage is detailed in Algorithm \ref{alg:LabelingProcedure}.  

\begin{algorithm}[!htb]
	\caption{\label{alg:LabelingProcedure}Two-Stage Labeling Procedure}
	\textbf{Input:} $X=\{x_{i}\}_{i=1}^{n}\subset\mathbb{R}^{D}$ (data), $\{p(x_{i})\}_{i=1}^{n}$ (empirical densities), $\{(\lambda_{i},\psi_{i})\}_{i=1}^{m}$ (eigenpairs of $P$), $\{(x_{i},y_{i})\}_{x_{i}\in \mathcal{L}}$ (labeled training data).\\
	In order of decreasing $p$-value among unlabeled points, assign each $x_{i}$ the same label as its $D_{t}$-nearest neighbor of higher $p$-value, unless the spatial consensus label of $x_{i}$ exists and differs, in which case $x_{i}$ is not labeled.\\
	Iterating in order of decreasing density among unlabeled points, assign each $x_{i}$ its consensus spatial label, if it exists, otherwise the same label as its $D_{t}$-nearest neighbor of higher $p$-value.  \\
	\textbf{Output:} $\{(x_{i},\hat{y}_{i})\}_{i=1}^{n}$ data points and estimated labels.

\end{algorithm}

The SR LAND algorithm consists in running Algorithms \ref{alg:ModeDetection}-\ref{alg:LabelingProcedure} in sequence.  Compared to existing diffusion geometric approaches to active learning \cite{Maggioni2019_Learning}, the proposed method differs in several key respects.  First, the underlying diffusion process is spatially regularized, so that the underlying random walk $P$ accounts for the high-dimensional structure of the data as well as its spatial structure.  Second, the labeling process in the proposed method consists of the two-stage process described in Algorithm \ref{alg:LabelingProcedure}, which further ensures smooth spatial labels of the data.  These two aspects of the proposed method are shown to be significant in Section \ref{sec:Experiments}.

\section{Experimental Results and Analysis}\label{sec:Experiments}

To empirically validate the efficacy of the proposed method, active learning experiments are performed on two benchmark real hyperspectral datasets: the Salinas A and Indian Pines datasets\footnote{\url{http://www.ehu.eus/ccwintco/index.php/Hyperspectral_Remote_Sensing_Scenes}}.  These datasets, together with their ground truth labels, are in Figures \ref{fig:SalinasASampling} and \ref{fig:IndianPines} respectively.

\begin{figure}[!htb]
\centering
\begin{subfigure}{.3\textwidth}
\includegraphics[width=\textwidth]{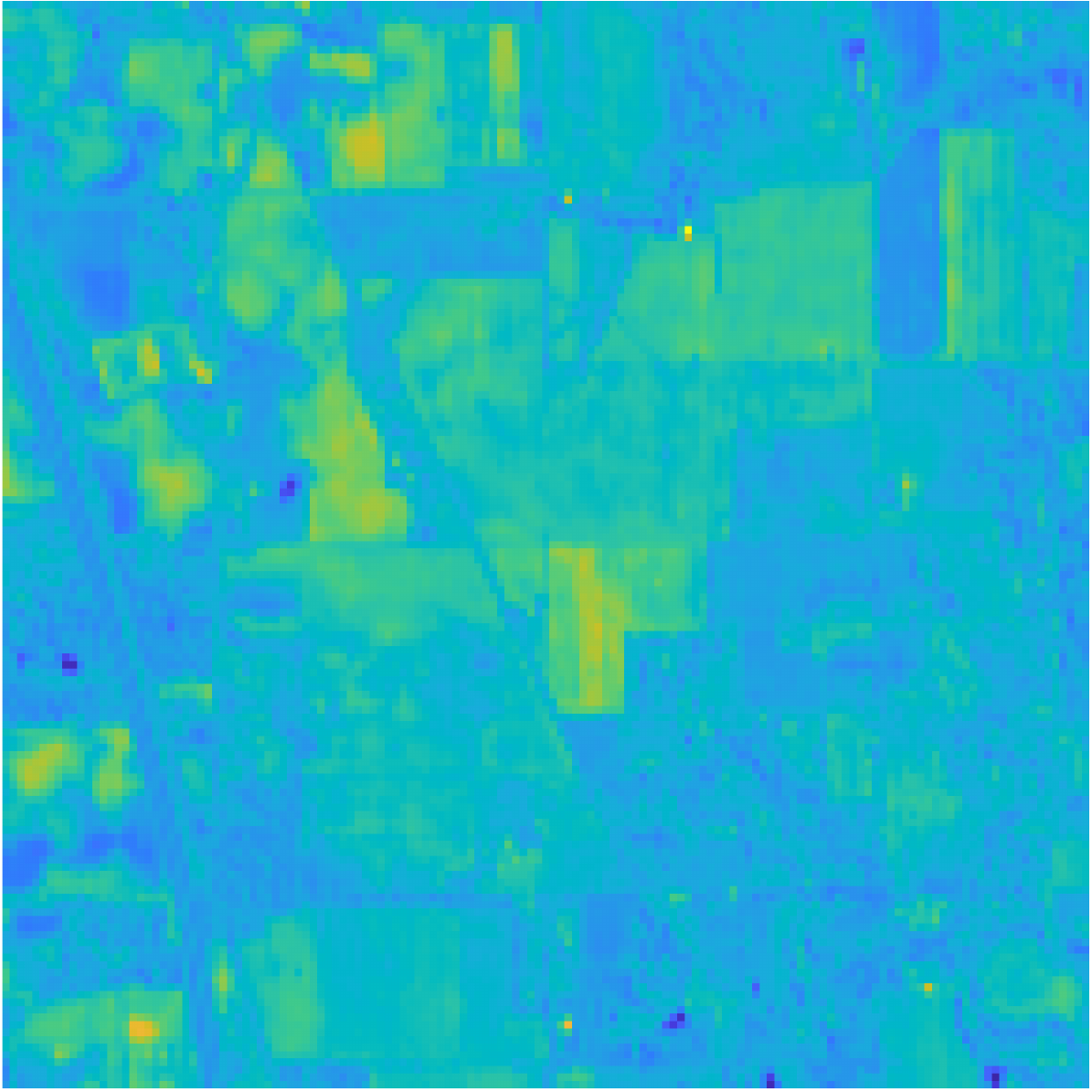}
\subcaption{Indian Pines}
\end{subfigure}
\begin{subfigure}{.33\textwidth}
\includegraphics[width=\textwidth]{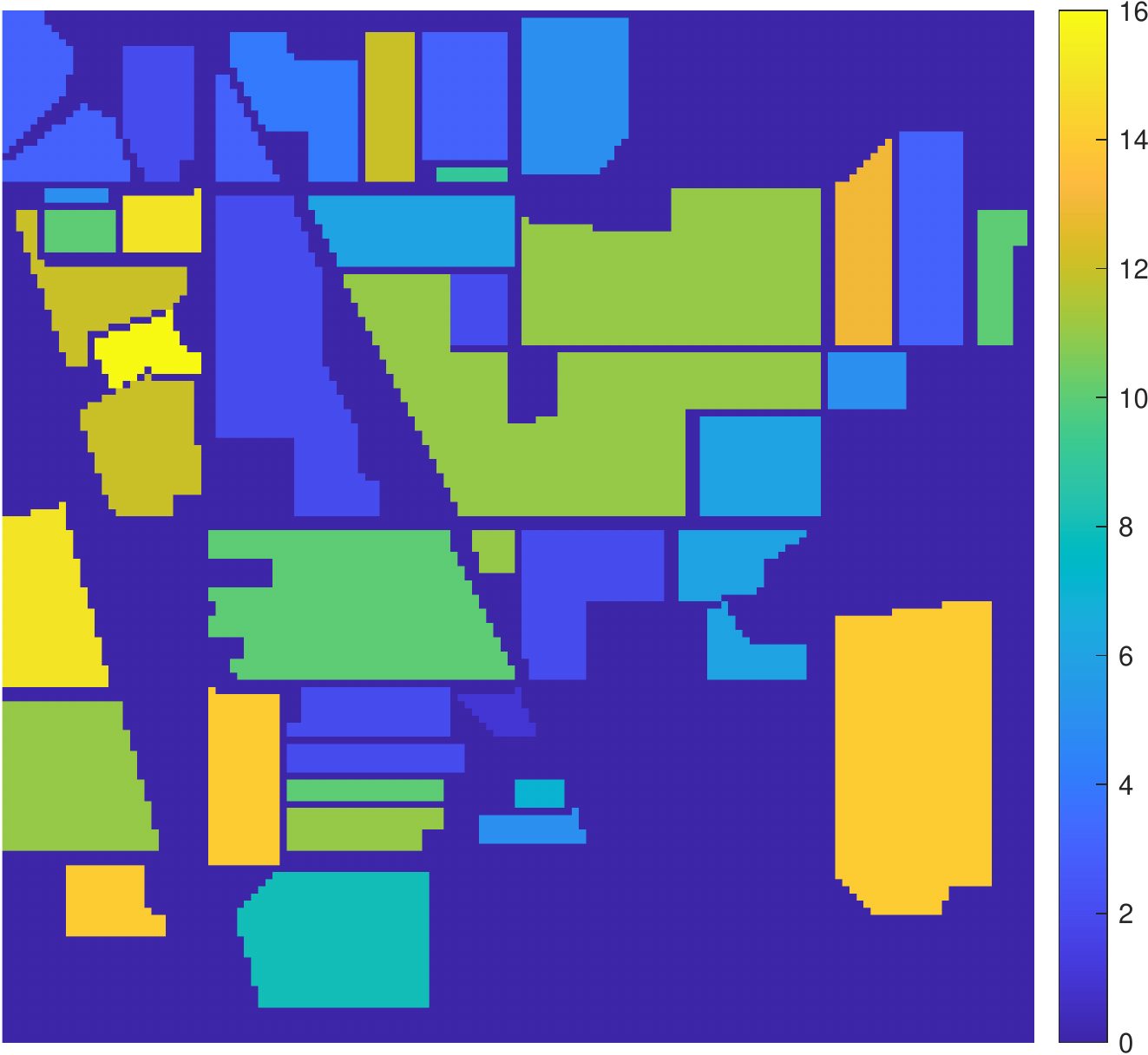}
\subcaption{Indian Pines Ground Truth}
\end{subfigure}
\begin{subfigure}{.32\textwidth}
\includegraphics[width=\textwidth]{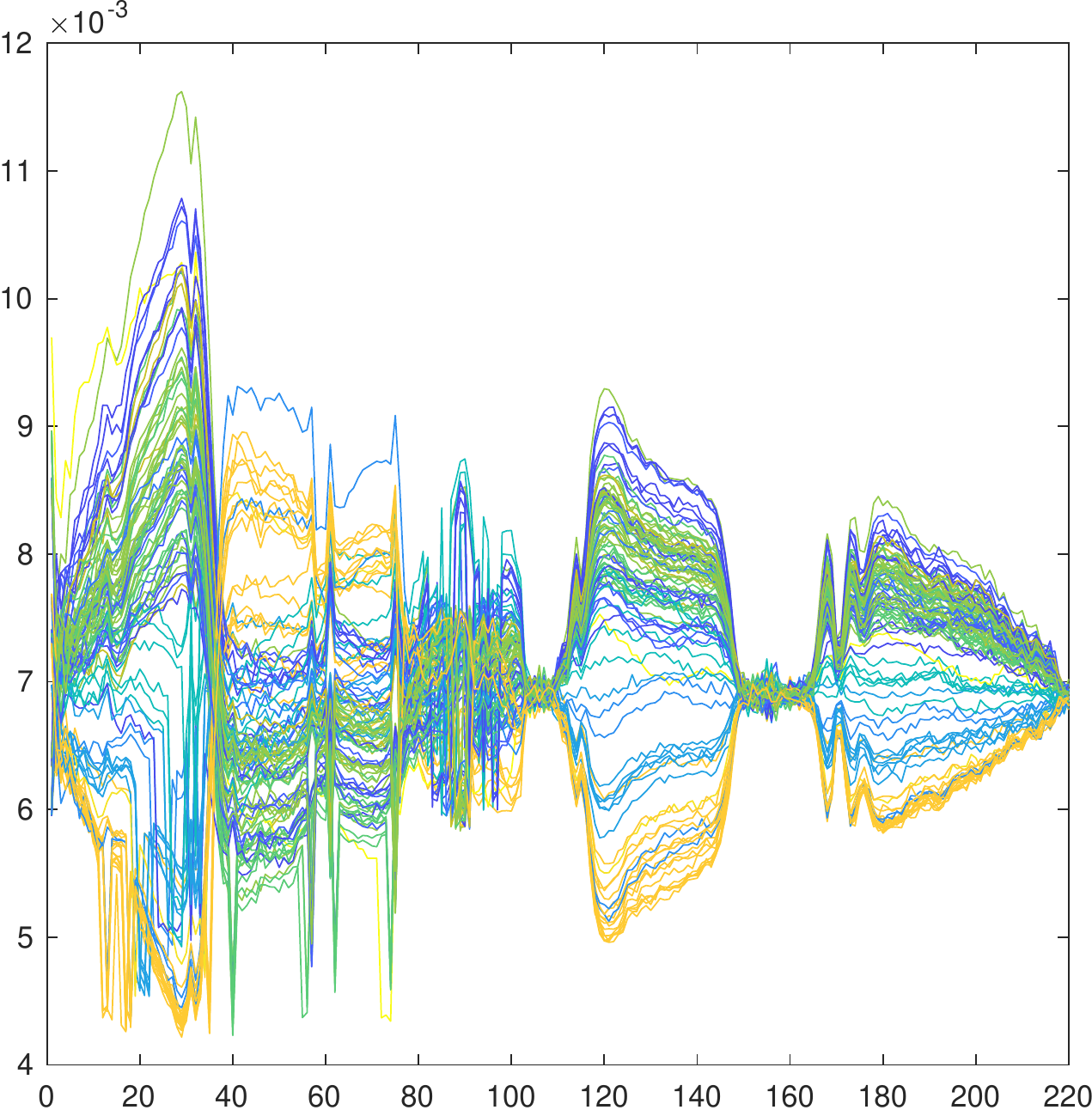}
\subcaption{Indian Pines Spectra}
\end{subfigure}
\caption{\label{fig:IndianPines}The Indian Pines data consists of $145\times 145$ spatial pixels, each of which is a 200 dimensional spectrum. The data was collected in IN, USA, and has a spatial resolution of 20m/pixel.  The sum of all spectral bands is shown in (a).  The full ground truth, arrayed spatially, is shown in (b).  A random sample of 100 spectra are shown in (c), with the same colormap as in (b).}
\end{figure}

The proposed method is compared against 12 related methods.  We benchmark against an SVM with random training samples (Random SVM), as well as an SVM with labels determined by cluster-based active learning \cite{Dasgupta2008_Hierarchical} (CBAL SVM).  Five state-of-the-art active learning algorithms for high-dimensional images are considered: SVM with samples generated by margin sampling (MS SVM) \cite{Tuia2009_Active}; SVM with samples generated by multiview sampling (MV SVM) \cite{Di2011_View}; SVM with samples generated by entropy query-by-bagging (EQB SVM) \cite{Tuia2009_Active}; SVM with image fusion and recursive filtering  (IFRF SVM) \cite{Kang2013_Feature}; and loopy back propagation (LBP) \cite{Li2012_Spectral}.  

We also compare against several variants of the proposed method, to illustrate the importance of its key innovations.  In particular, we consider variants in which random labels are used in the LAND algorithm (SR Random LAND) and in which only the boundary is sampled (SR Boundary LAND).  The boundary queries are determined by which points are most ambiguous with respect to their nearest modes.  More precisely, to estimate the cluster boundaries, we calculate for each  $x_{i}$ the distance between its two $D_{t}$-nearest modes $x_{i,1}^{*}, x_{i,2}^{*}$.  If this distance is small, it suggests $x_{i}$ is in between two clusters, and the label of $x_{i}$ is ambiguous and should be queried.  The associated function whose minimizers are near the estimated boundaries is: \begin{equation}\label{eqn:BoundaryCondition}\mathcal{B}(x_{i})=|D_{t}(x_{1,i}^{*},x_{i})-D_{t}(x_{2,i}^{*},x_{i})|, \ x_{1,i}^{*}, x_{2,i}^{*} \text{ the } D_{t}\text{-nearest modes to } x_{i}.\end{equation}  This approach is related to a previously developed method for active sampling of hyperspectral images \cite{Murphy2019_Unsupervised, Murphy2018_Iterative}, and may be understood as a margin-driven active learning method.  Indeed, the points which minimize $\mathcal{B}$ are in some sense at the boundaries of the cluster structure learned from the diffusion geometry of the data; sampling these points has the impact of setting the boundary between the different clusters.  In order to distinguish between these methods, the proposed method (as described in Section \ref{sec:Algorithm}) is called SR Core LAND in this section, since it samples actively from the estimated cluster cores.  In order to investigate the significance of the spatial regularization of the underlying diffusion process, we also consider all LAND variants, but without spatial regularization (denoted Random LAND, Boundary LAND, and Core LAND, respectively).  

To evaluate accuracy, we consider the proportion of correctly labeled points, denoted the overall accuracy.  Alternative notions of accuracy (e.g. classwise average accuracy and Cohen's $\kappa$ statistic) were also computed, but are not shown for reasons of space.  To account for randomness in the methods that are random, the results of 10 independent trials were averaged.

Results as a function of the number of active queries are shown in Figure \ref{fig:ResultsPlots}.  SR Core LAND used a spatial radius $R=11$ for Salinas A and $R=14$ for Indian Pines.  Results for specific numbers of active samples---corresponding to a small and large number of active samples---appear in Table \ref{tab:Results}, with corresponding images of the labeling results appearing in Figure \ref{fig:LabelPlots}.  We see that SR Core LAND quickly improves as a function of the number of training samples.  In particular, on Indian Pines, it strongly outperforms all comparison methods, achieving more than $86\%$ labeling accuracy with only 100 active queries, which is less than .5\% of the total points in the dataset.  We note that SR Core LAND does not guarantee a balanced sample from each class; we merely require that, like all the comparison methods, each class has at least one label in the training set, which makes its performance all the more impressive.  On Salinas A, SR Core LAND achieves near-perfect accuracy with fewer than 20 labels, less than .3\% of the total number of points in the dataset.  On this simpler dataset, several comparison methods also perform well, though SR Core LAND is the optimal performer in this small training set regime.  
 
\begin{figure}[!htb]
\centering
\begin{subfigure}{.49\textwidth}
\includegraphics[width=\textwidth]{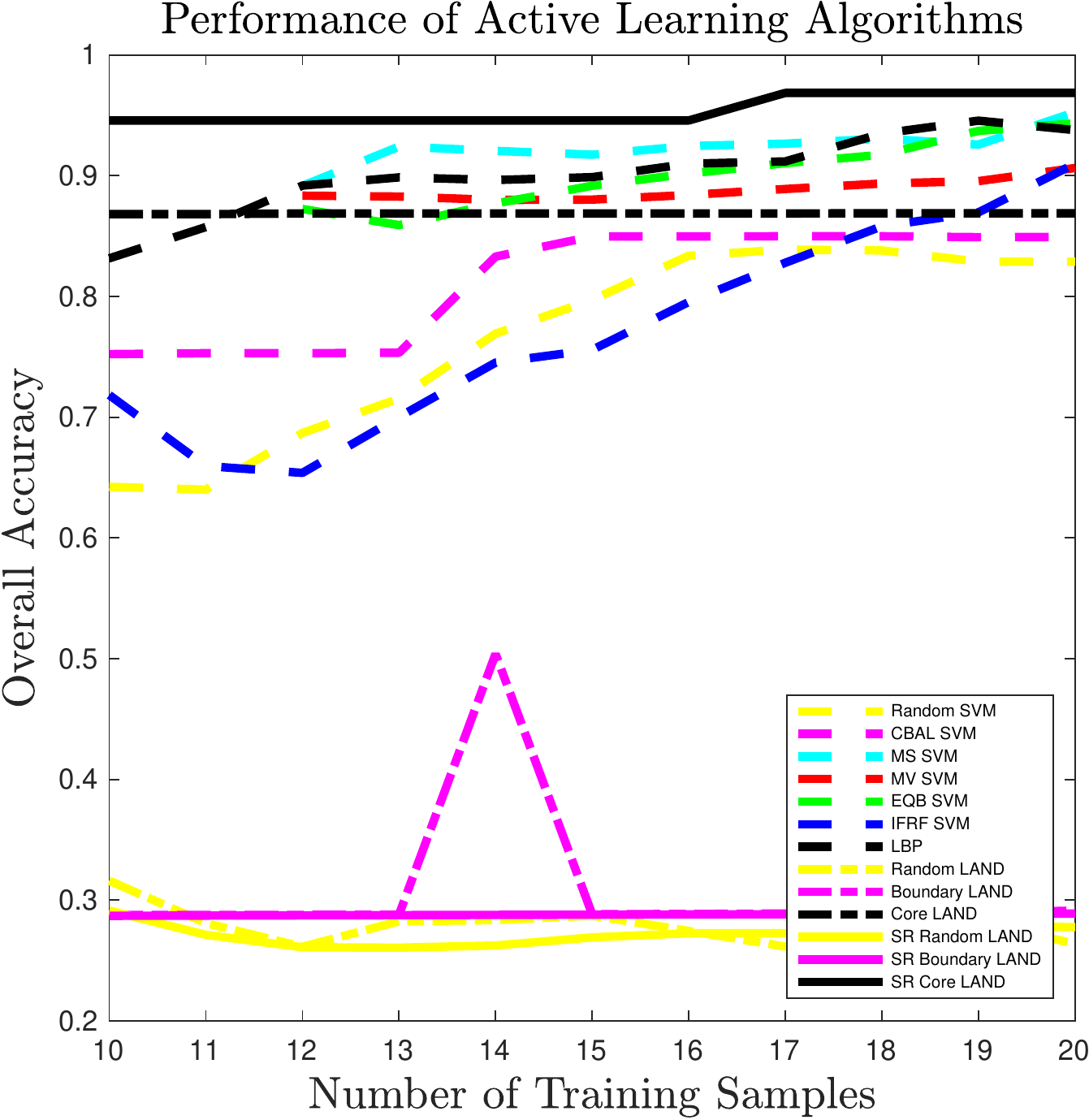}
\subcaption{Results for Salinas A}
\end{subfigure}
\begin{subfigure}{.49\textwidth}
\includegraphics[width=\textwidth]{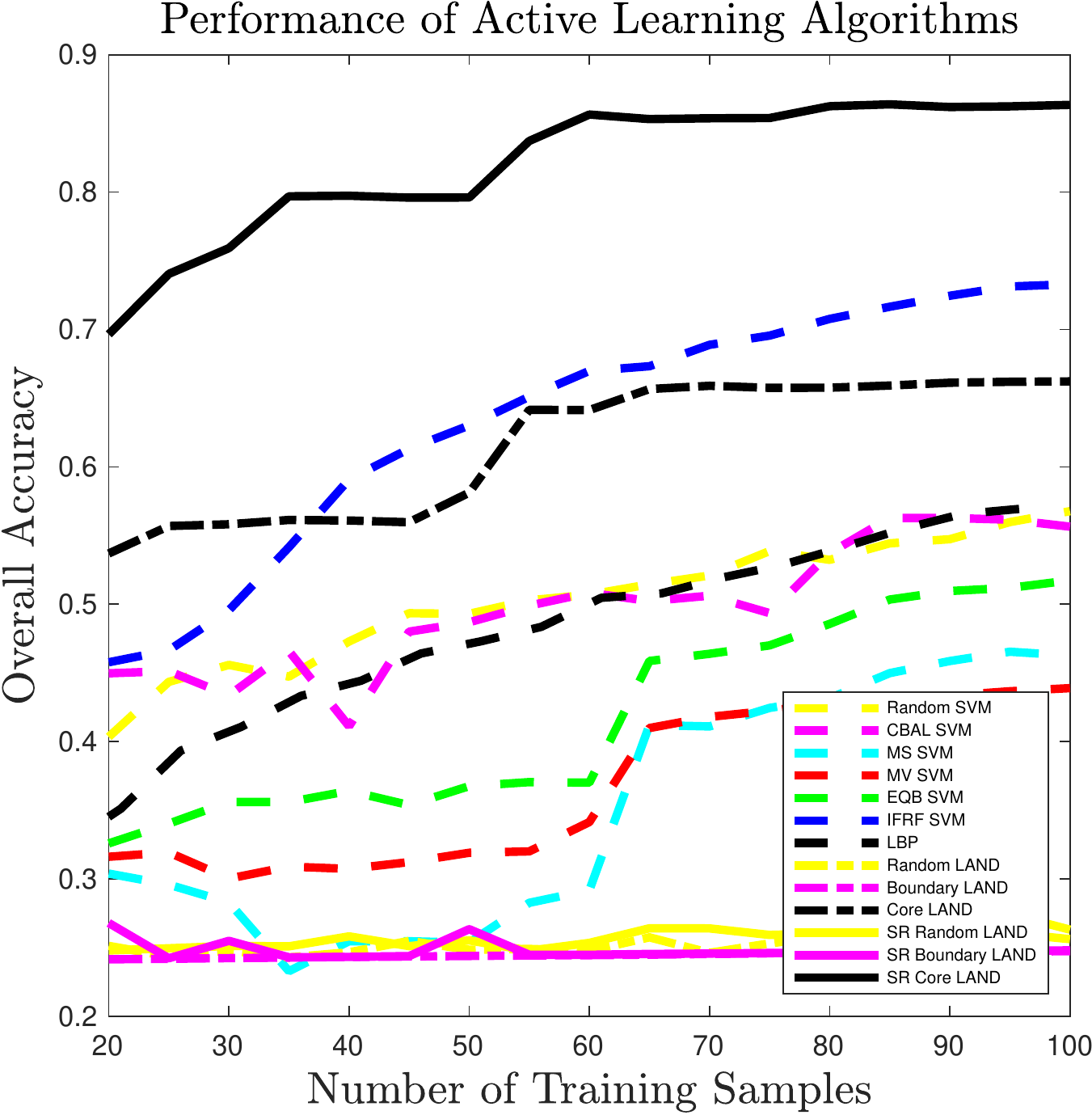}
\subcaption{Results for Indian Pines}
\end{subfigure}
\caption{\label{fig:ResultsPlots}Results as a function of the number of actively queried samples.  We see that the proposed SR Core LAND method performs best, and in particular rapidly achieves very strong performance with respect to a small number of training labels.  The second best performer changes depending on the dataset and number of query points, but in general, the non-baseline SVM methods and Core  LAND tend to perform well.  The Boundary and Random LAND algorithms perform very poorly, demonstrating the need to incorporate at least some cluster core information into the LAND approach.  Note that some methods are not monotonic increasing in the number of queries (e.g. CBAL SVM and MS SVM).}
\end{figure}

\begin{figure}[!htb]
\centering
\begin{subfigure}{.49\textwidth}
\includegraphics[width=\textwidth]{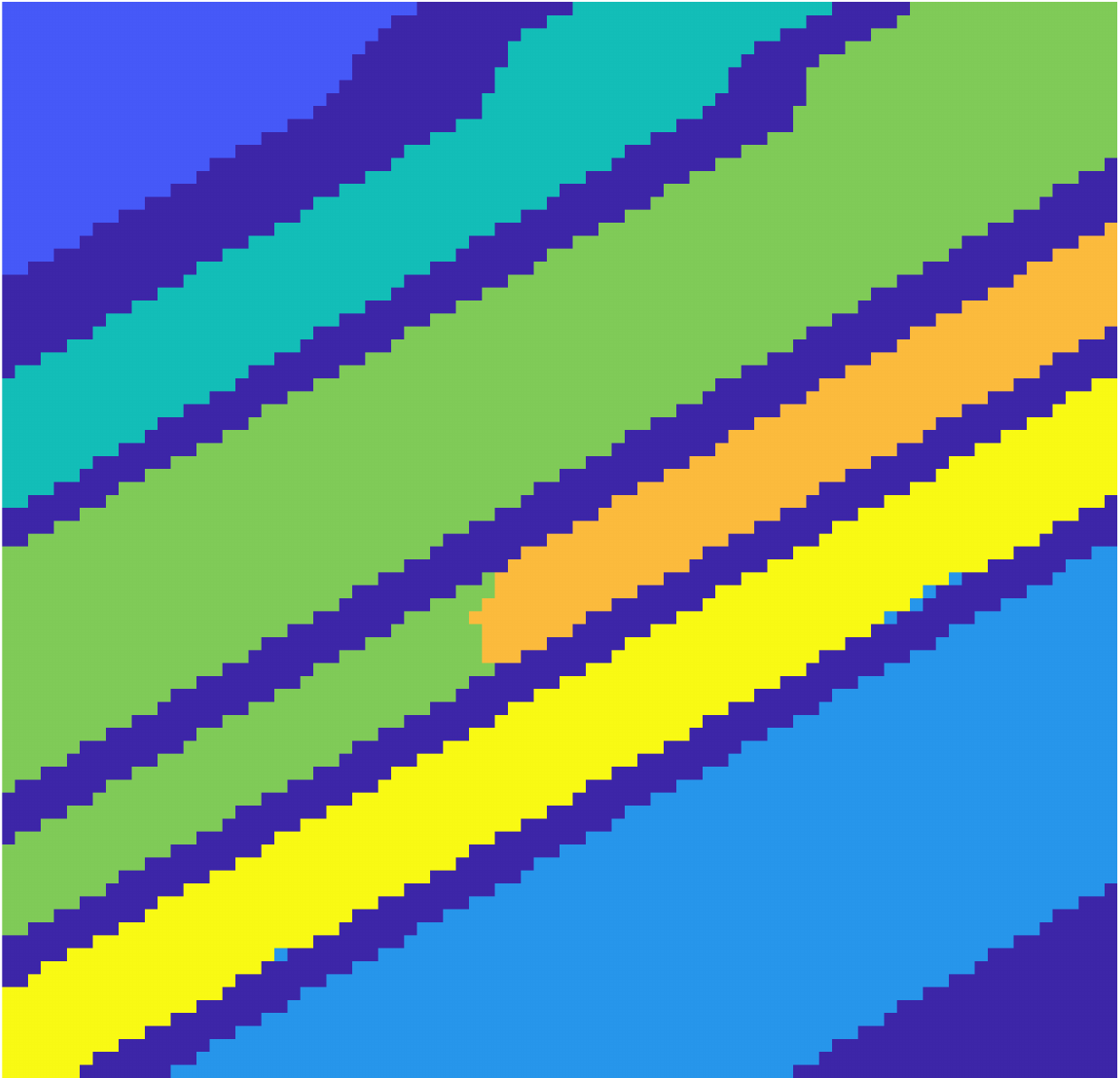}
\subcaption{Salinas A, 10 active queries}
\end{subfigure}
\begin{subfigure}{.49\textwidth}
\includegraphics[width=\textwidth]{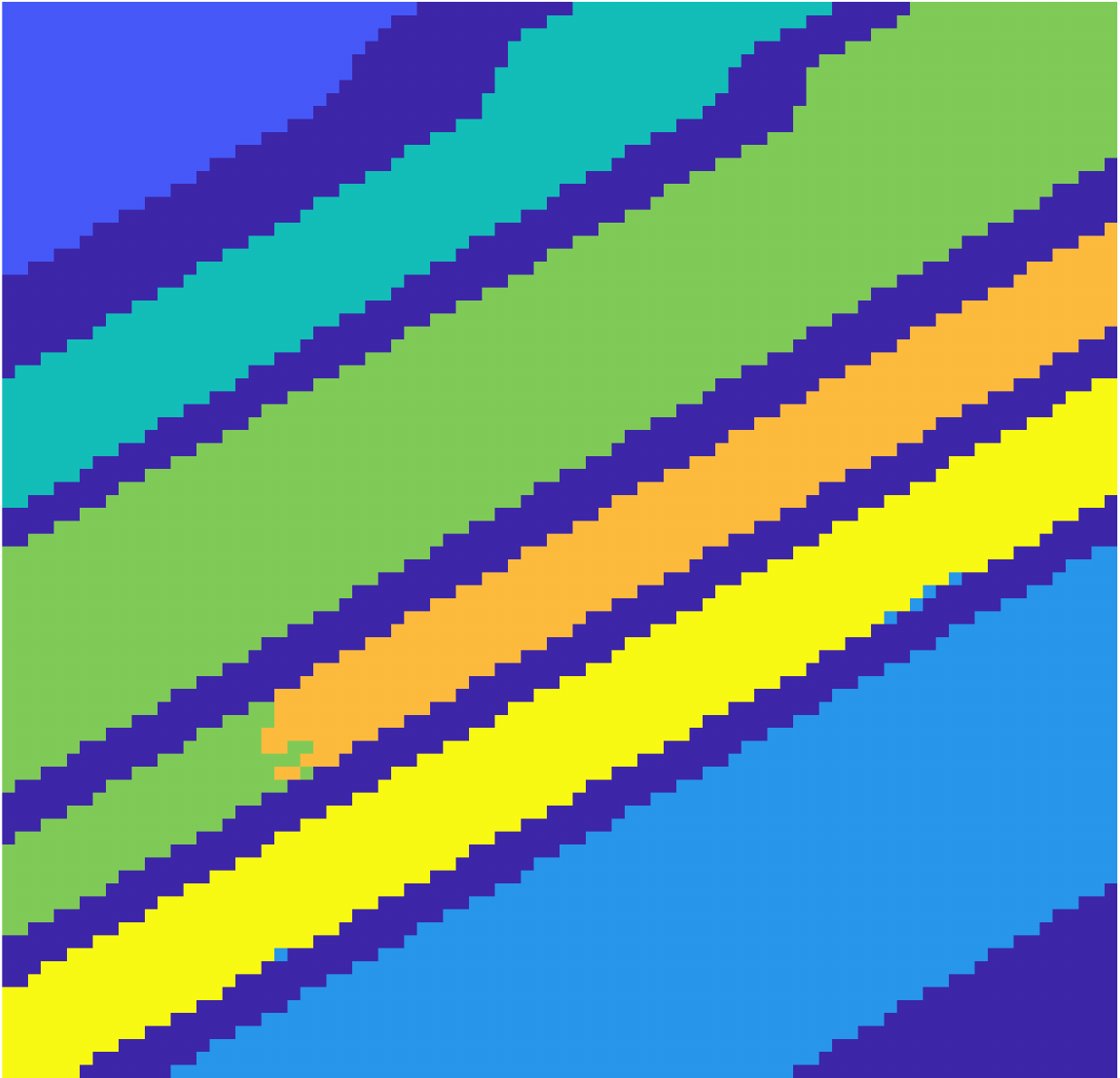}
\subcaption{Salinas A, 20 active queries}
\end{subfigure}
\begin{subfigure}{.49\textwidth}
\includegraphics[width=\textwidth]{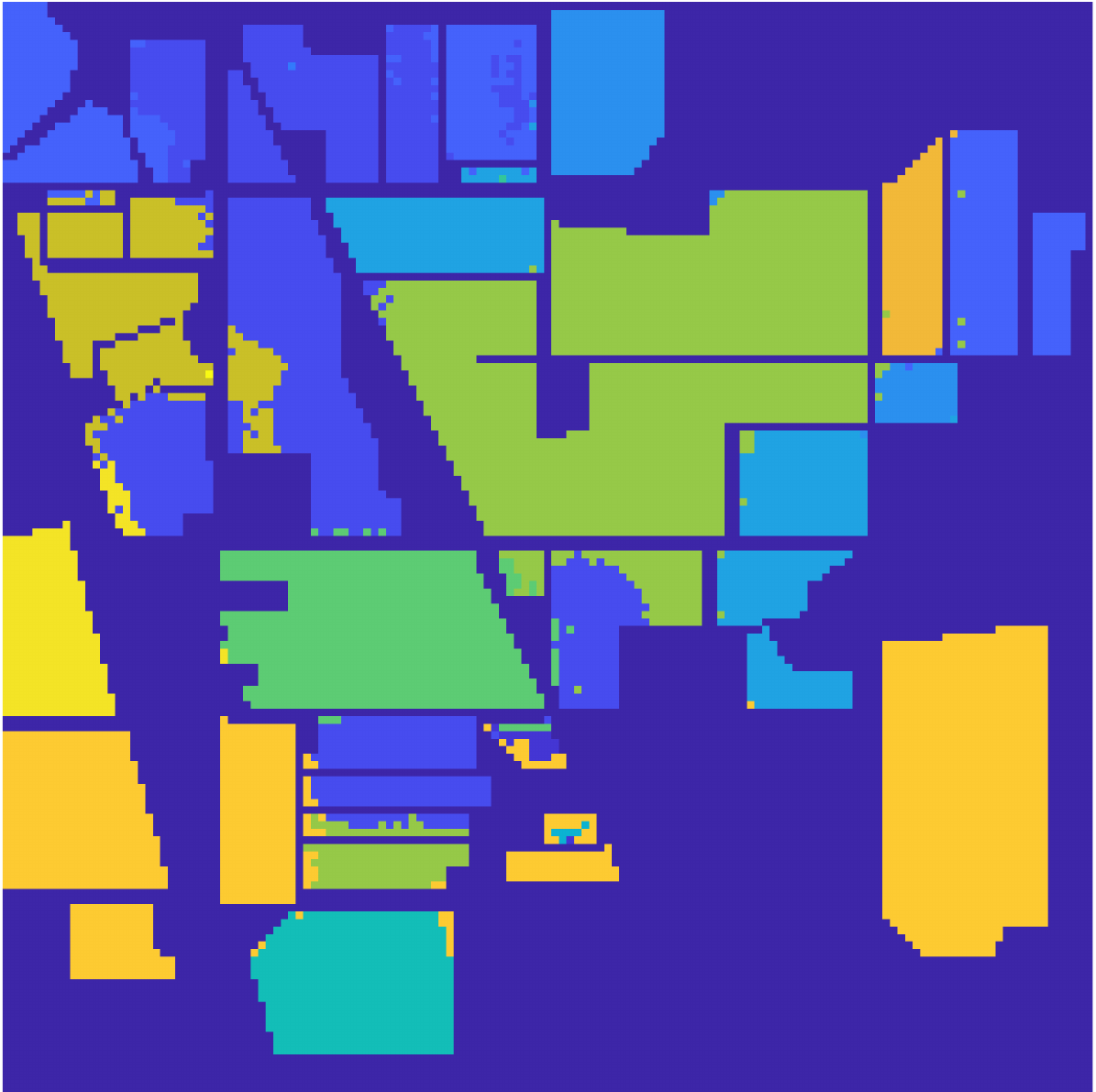}
\subcaption{Indian Pines, 50 active queries}
\end{subfigure}
\begin{subfigure}{.49\textwidth}
\includegraphics[width=\textwidth]{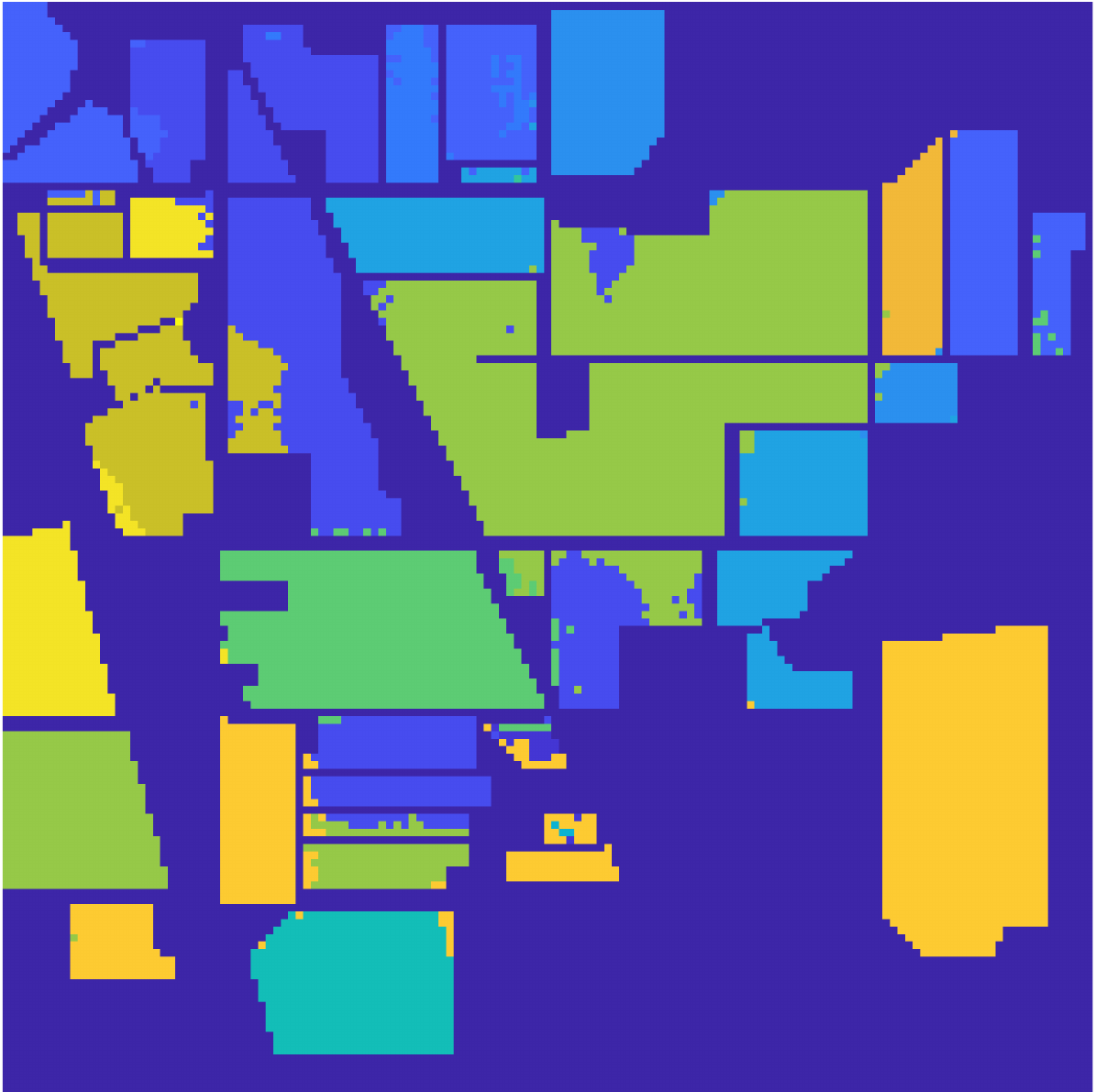}
\subcaption{Indian Pines, 100 active queries}
\end{subfigure}
\caption{\label{fig:LabelPlots}Results of the SR LAND scheme are shown for Salinas A in (a), (b) and for Indian Pines in (c), (d).  Comparing to the respective ground truths (see Figure \ref{fig:SalinasASampling} (b) and Figure \ref{fig:IndianPines} (b), respectively), the impact of adding additional queries is to improve results not incrementally, but dramatically by correctly labeling entire classes.  Indeed, the characterization of classes in terms of their modes suggests that all points that are close in diffusion distance from their nearest mode should belong to the same class.}
\end{figure}

\begin{table}[htb!]
\centering
\begin{adjustbox}{max width=\textwidth}
\begin{tabular}{| c | c | c | c | c | c | c  | c | c | c | c | c | c | c | }\hline
Method & Random SVM & CBAL SVM & MS SVM & MV SVM & EQB SVM & IFRF SVM & LBP & Random LAND & Boundary LAND & Core LAND & SR Random LAND & SR Boundary LAND & SR Core LAND \\ \hline
Salinas A (10 Labels) & 0.6988  &  0.7835  &  0.8672  &  0.8672  &  \underline{0.8943}  &  0.8240  &  0.6907  &  0.2865  &  0.2883  &  0.8682  &  0.2866  &  0.2870  &  \textbf{0.9458} \\ \hline
Salinas A (20 Labels) & 0.8764  &  0.8502  &  0.9618  &  0.9419  &  \underline{0.9677}  &  0.9295  &  0.9192  &  0.2876  &  0.2893  &  0.8689  &  0.2898  &  0.2889  &  \textbf{0.9686} \\ \hline
Indian Pines (50 Labels)   &  0.5138  &  0.4744  &  0.3273  &  0.3936  &  0.4274  &  \underline{0.6520}  & 0.4898   &  0.2445  &  0.2445  &  0.6429  &  0.2436  &  0.2448  &  \textbf{0.8389} \\ \hline
Indian Pines (100 Labels)  &  0.5583  &  0.4868  &  0.3908  &  0.4415  &  0.4813  &  \underline{0.7412}  & 0.5722  &  0.2763  &  0.2474  &  0.6622  &  0.2939  &  0.2480  &  \textbf{0.8648}\\ \hline
\end{tabular}
\end{adjustbox}
\caption{\label{tab:Results}Summary of performance of different methods in terms of overall accuracy.  We see that for all experiments, the SR Core LAND algorithm performs optimally among all comparison methods.  We notice also that with only 10 labels, SR Core LAND achieves nearly $95\%$ accuracy on Salinas A.  Indeed, the unsupervised variant LUND \cite{Maggioni2019_Learning_JMLR, Murphy2019_Unsupervised} achieves roughly 85\% accuracy with no labels, suggesting that the diffusion core structure learned is highly salient.  Similarly, with only 100 labels in a dataset with 16 classes, SR Core LAND achieves over 86$\%$ accuracy on Indian Pines, far outpacing any competitor methods.}
\end{table}

Figure \ref{fig:SpatialParameterPlots} shows the impact of the spatial diffusion radius $R$ on SR LAND.  Like a regularization parameter, the optimal choice is not too big (insufficient spatial regularization) and not too small (too much spatial regularization).

\begin{figure}[!htb]
\centering
\begin{subfigure}{.49\textwidth}
\includegraphics[width=\textwidth]{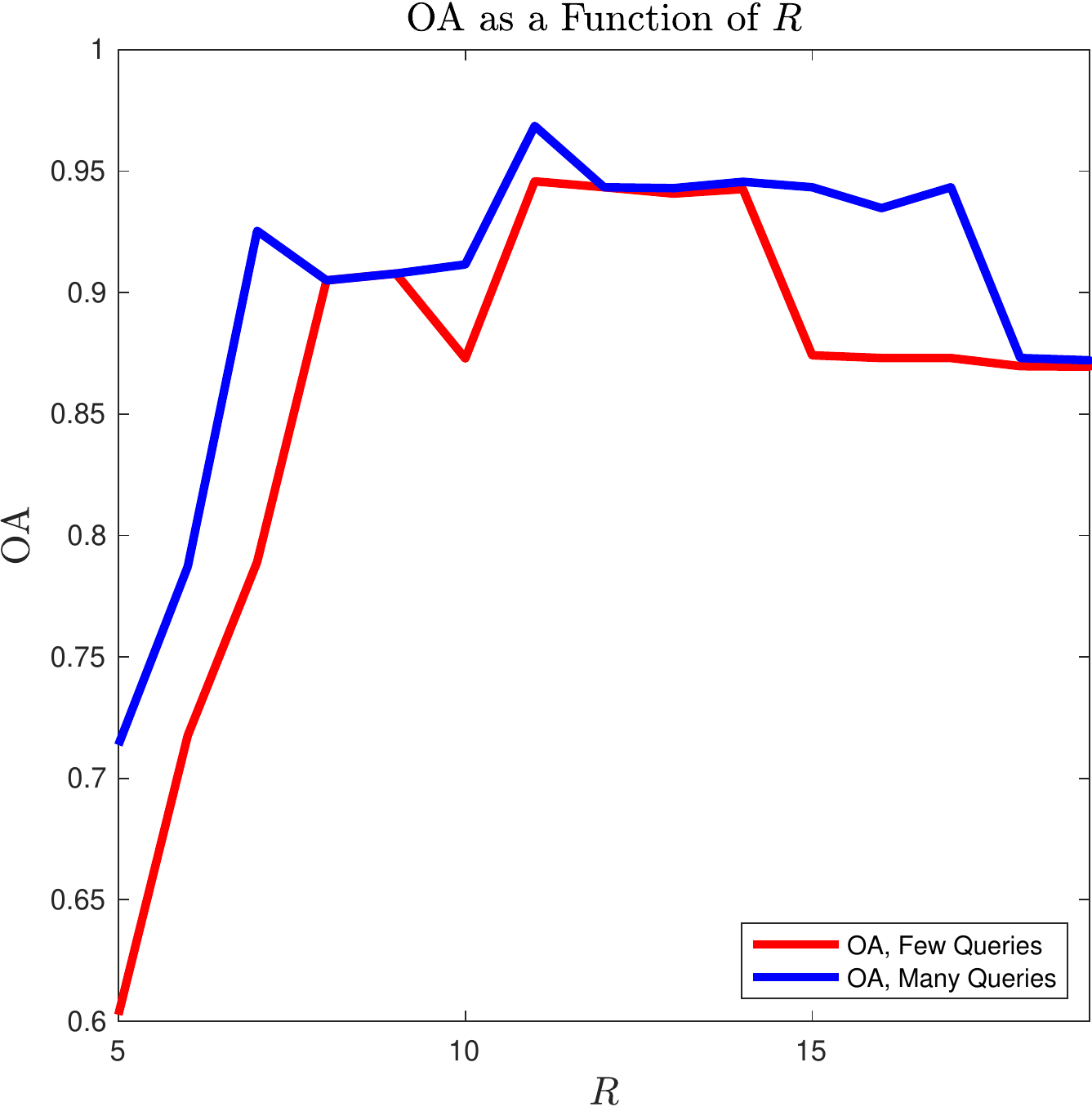}
\subcaption{Salinas A}
\end{subfigure}
\begin{subfigure}{.49\textwidth}
\includegraphics[width=\textwidth]{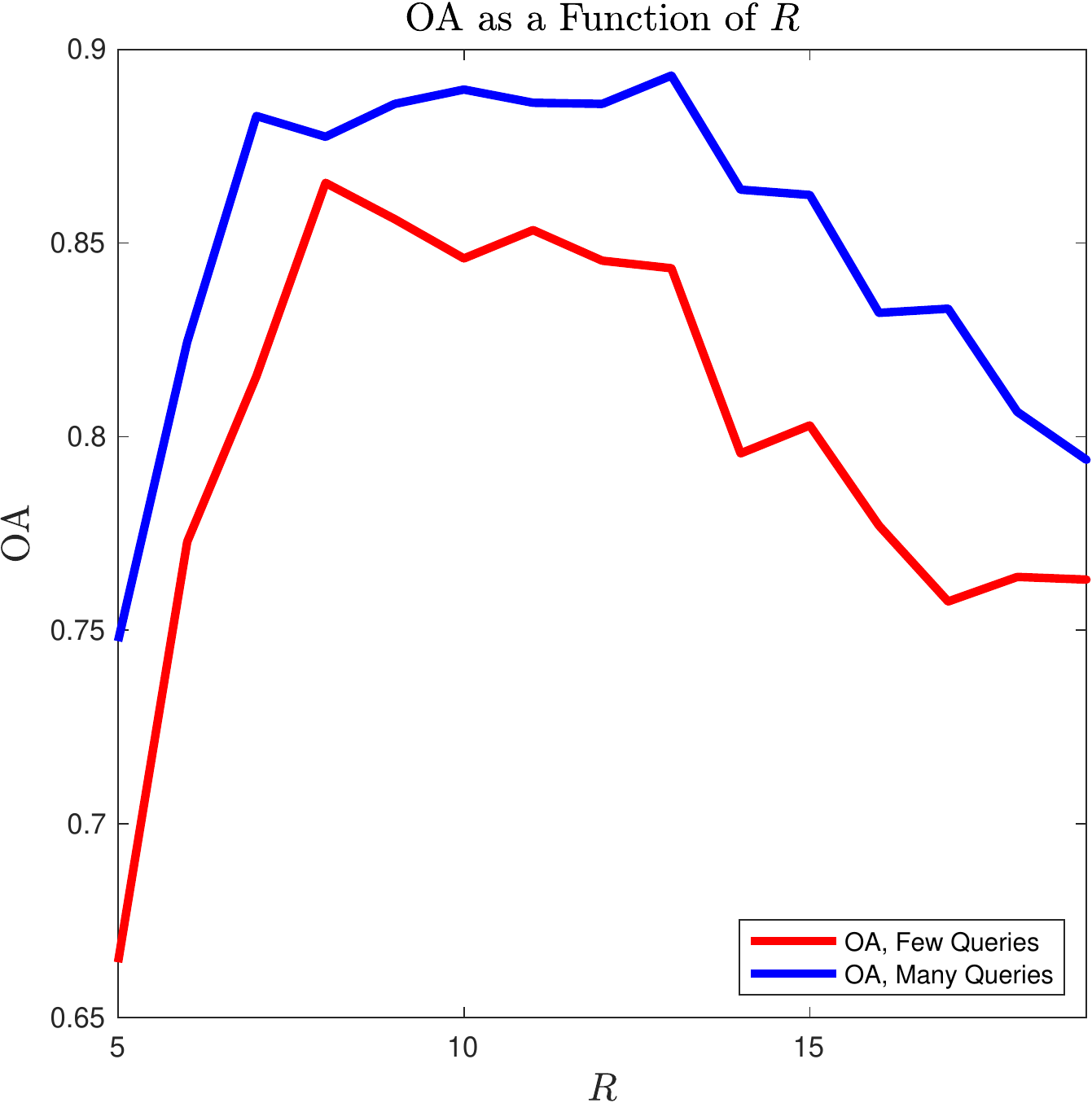}
\subcaption{Indian Pines}
\end{subfigure}
\caption{\label{fig:SpatialParameterPlots}Impact of spatial radius $R$ of the diffusion process.  We see for both datasets, there is a general trend for results to improve as the radius increases from a small value, before reaching a maximum then decreasing.  We remark that the experiments were run for both a smaller and larger number of active queries.  For Salinas A, 10 and 20 samples are used.  For Indian Pines, 50 and 100 samples are used.}
\end{figure}

\subsection{Complexity Analysis and Runtime}\label{subsec:ComplexityRuntime}

The empirical runtimes of all methods are in Table \ref{tab:RunTime}.  The proposed algorithm enjoys low computational complexity with respect to both the ambient data dimensionality $D$ and the number of points $n$, under a mild regularity assumption on the data in terms of the number of high density points:

\begin{defn}Data $\{x_{i}\}_{i=1}^{n}\subset\mathbb{R}^{D}$ with associated empirical densities $\{f_{\sigma}(x_{i})\}_{i=1}^{n}$ and diffusion distance function $D_{t}:X\times X\rightarrow [0,1)$ satisfies the \emph{few density peaks hypothesis (FDP)} if all except for $O(\log(n))$ data points have a point of higher empirical density among their $O(\log(n))$ $D_{t}$-nearest neighbors.  

\end{defn}

Note that if the data is generated from a mixture of distributions with finitely many local density maxima (e.g. a mixture of Gaussians), then the FDP hypothesis holds in the asymptotic limit $n\rightarrow\infty$.  This suggests it holds in the finite $n$ case with high probability for $n$ sufficiently large \cite{Abraham2004_Asymptotic, Jiang2017_Uniform}.  The FDP hypothesis guarantees that, for data with intrinsically low dimension in the sense of doubling dimension \cite{Beygelzimer2006_Cover}, SR Core LAND is fast:

\begin{thm}Suppose data $\{x_{i}\}_{i=1}^{n}\subset\mathbb{R}^{D}$ satisfies the FDP hypothesis and has intrinsic dimensionality $d$ in the sense of doubling dimension.  If $r,m=O(1), k=O(\log(n))$, then the proposed SR Core LAND algorithm has complexity that is quasilinear in $n$ and $D$ and exponential in $d$.

\end{thm}

\begin{proof}The construction of the transition matrix $P$ is $O(r^{2}n)$, and since $r=O(1)$, this $O(n)$.  In order to use (\ref{eqn:DiffusionDistancesSpectralFormulationTruncated}), we must compute the $m$ eigenvectors of $P$ with largest (in modulus) eigenvalues.  Since $P$ is sparse (each row of $P$ has at most $r^{2}$ non-zero entries), this can be done in complexity $O(m^{2}n)=O(n)$.  The empirical density estimate requires a $k$-nearest neighbor search in $\mathbb{R}^{D}$, which has complexity $O(n\log(n)^{2}DC_{d})$ using the cover trees nearest neighbors algorithm \cite{Beygelzimer2006_Cover}, where $C_{d}$ is a constant exponential in $d$, the doubling dimension of the underlying data.  Under the FDP hypothesis, the mode detection stage of the algorithm has complexity $O(n\log(n)D)$.  Again by the FDP hypothesis, the labeling stage has complexity $O(n\log(n)D)$.  For each point, determining its spatial nearest neighbors is $O(r^{2})=O(1)$, so all spatial information can be determined with complexity $O(n)$.  Thus, the overall procedure has complexity $O(n\log(n)^{2}DC_{d})$.
\end{proof}

We remark that the theoretical performance of a much simplified version of the proposed method was analyzed \cite{Maggioni2019_Learning} and shown to behave well as a function of the number of active queries.  The proposed version, which directly incorporates spatial regularity into the underlying diffusion process and samples near both cluster cores and cluster boundaries, substantially improves empirical results, though a theoretical analysis is beyond the scope of the present article.

\begin{table}[thb!]
\centering
\begin{adjustbox}{max width=\textwidth}
\begin{tabular}{| c | c | c | c | c | c | c  | c | c | c | c | c | c | c | }\hline
Method & Random SVM & CBAL SVM & MS SVM & MV SVM & EQB SVM & IFRF SVM & LBP & Random LAND & Boundary LAND & Core LAND & SR Random LAND & SR Boundary LAND & SR Core LAND \\ \hline
Salinas A & 0.87 & 18.40 & 3.23 & 20.35 & 6.35 & 2.31 & 17.68 & 10.01 & 9.88  & 10.06 & 13.80 & 13.65 & 13.84   \\ \hline
Indian Pines & 1.18 & 122.52 & 6.14 & 23.65 & 14.28 & 3.38 & 44.21 & 69.41 & 68.91  & 76.30 & 113.61 & 112.92 &  114.61  \\ \hline
\end{tabular}
\end{adjustbox}
\caption{\label{tab:RunTime}The run time in seconds of the different algorithms.  The spatially regularized LAND algorithms are somewhat slower than their non-spatially regularized variants.  Among the comparison methods, many of the SVM methods are faster than the LAND methods, because the former are written in C and the latter in MATLAB.}
\end{table}

\section{Conclusions and Future Research}\label{sec:Conclusions}

This article demonstrates the value in using regularized diffusion processes for active learning.  In particular, the impact of sampling near the estimated cores of the inferred data clusters generated from the spatially regularized random walk is shown to significantly improve not only over state-of-the-art methods for active learning of high-dimensional images, but also over variants of the proposed method that sample near boundaries or without spatial regularization.  We remark that in the case that the underlying data lacks spatial smoothness, either due to highly localized classes or highly diffuse classes, the proposed method may not confer an advantage.  

It is of interest to develop a mathematical model which quantifies the tradeoff between sampling near the cores and sampling near the boundaries.  In the context of image processing, a spatial regularity constraint is more valuable the smoother the class labels are spatially.  Quantifying such a notion of spatial regularity in combination with a cluster model on the image spectra would shed light on the efficacy of the proposed method.  

One may also wish to iteratively refine the boundary queries so that they do not always aggregate in the most ambiguous part of the image; this localization of cluster queries is potentially the reason for the poor performance of the Boundary LAND approach.  While it is certainly valuable to have many samples near the most ambiguous region, there may be situations in which it is more optimal to have a smaller number of boundary queries in the most ambiguous region so that other ambiguous boundaries may be queried as well.  Developing methods for incorporating some boundary points to improve labeling accuracy in this setting is the topic of ongoing inquiry.

\section*{Acknowledgements}  We are grateful to Alex Cloninger and Mauro Maggioni for many helpful discussions related to this work.  This research is partially supported by the US National Science Foundation grants NSF-DMS 1912737 and NSF-DMS 1924513.


\bibliographystyle{unsrt}
\bibliography{SR_LAND.bib}

\end{document}